\documentclass[11pt]{article}
\usepackage{amsmath,amssymb}
\usepackage{latexsym}
\usepackage{graphicx}
\usepackage{algorithm}
\usepackage{amsthm}
\usepackage{authblk}

\newtheorem{definition}{Definition}
\newtheorem{theorem}[definition]{Theorem}

\newtheorem{lemma}[definition]{Lemma}
\newtheorem{proposition}[definition]{Proposition}
\newtheorem{example}[definition]{Example}
\newtheorem{remark}[definition]{Remark}

\newcommand{\lnon}{\overline}

\newcommand{\cvec}{{\bf c}}
\newcommand{\xvec}{{\bf x}}

\newcommand{\tvec}{{\bf t}}
\newcommand{\ds}{\displaystyle}
\newcommand{\twopartdef}[4]
{	\left\{
	\begin{array}{ll}
		#1 & \mbox{if } #2 \\
		#3 & \mbox{if } #4
	\end{array}
	\right.
}
\newcommand{\threepartdef}[6]
{	\left\{
	\begin{array}{ll}
		#1 & \mbox{if } #2 \\
		#3 & \mbox{if } #4  \\
		#5 & \mbox{if } #6
	\end{array}
	\right.
}

\begin{document}

\title{On the Trade-off between the Number of Nodes and the Number of Trees
in a Random Forest} 

\author[1]{Tatsuya Akutsu\thanks{Corresponding author.
Partially supported by JSPS KAKENHI \#JP22H00532 and \#JP22K19830.}}
\author[2]{Avraham A. Melkman}
\author[3]{Atsuhiro Takasu}
\affil[1]{Bioinformatics Center, Institute for Chemical Research, Kyoto University, Japan}
\affil[2]{Department of Computer Science, Ben-Gurion University of the Negev, Israel}
\affil[3]{National Institute of Informatics, Chiyoda-ku, Tokyo, Japan}

\maketitle

\begin{abstract}
In this paper, we focus on the prediction phase of a random forest and
study the problem of representing a bag of decision trees
using a smaller bag of decision trees,
where we only consider binary decision problems on the binary domain
and simple decision trees in which 
an internal node is limited to querying 
the Boolean value of a single variable.
As a main result,
we show that the majority function of $n$ variables
can be represented by a bag of $T$ ($< n$) decision trees
each with polynomial size if $n-T$ is a constant,
where $n$ and $T$ must be odd (in order to avoid the tie break).
We also show that a bag of $n$ decision trees
can be represented by a bag of $T$ decision trees
each with polynomial size if $n-T$ is a constant and
a small classification error is allowed.
A related result on the $k$-out-of-$n$ functions is presented too.

\end{abstract}

\section{Introduction}

One of the major machine learning models is a random forest.
The model was proposed by Breiman in 2001 \cite{breiman01},
and is represented as a bag of trees.
When used for classification tasks its decision is the majority vote of the outputs of classification trees,
whereas for regression tasks
its prediction is the average over the outputs of 
regression trees.
In spite of the simplicity of the model,
its predictive accuracy is high in practice. Consequently, it
has been applied to a number of classification and regression tasks
\cite{delgado14}, and it remains a strong option even 
as various kinds of deep neural networks (DNN) have been put forward
for these tasks.
For example, Xu et al. compared DNNs and decision forests
and reported that deep forests achieved better results than DNNs when the training
data were limited \cite{xu21}.
Random forests are also used as a building block for deep 
%%AM.23.12.21>
%networks.
%Zhou and Feng proposed a method to exploit random forests in deep networks
networks, e.g.
%%<AM.AM.23.12.21
\cite{zhou17}.
Because of its strong predictive power,
the learnability of the random forest has been extensively studied.
For example, Breiman
showed that the generalization error of a random forest
depends on the strength of the individual trees in the forest and
their correlations \cite{breiman01}.
Lorenzen et al. analyzed the generalization error \cite{lorenzen19}
from a PAC (Probably Approximately Correct) learning viewpoint \cite{valiant84}.
Biau showed that the rate of convergence depends only on the number of strong
features \cite{biau12}.
Oshiro et al. empirically studied the relations between the number of trees
and the prediction accuracy \cite{oshiro12}.
Recently, Audemard et al. studied the complexity of finding small explanations 
on random forests \cite{audemard22}.

From a theoretical computer science perspective,
extensive studies have been done on the circuit complexity 
\cite{amano18,engels20,goldmann92,kulikov19,siu95,testa19}
and the decision tree complexity \cite{chistopolskaya22,magniez16} for
realizing the majority function, and more generally, linear threshold functions.
For example,
Chistopolskaya and Podolskii studied lower bounds on the size of the decision
tree for representing the threshold functions and the parity function \cite{chistopolskaya22}.
In particular, they showed an $n-o(1)$ lower bound when each node in a decision tree is a Boolean function
of two input variables,
where $n$ is the number of variables.
Extensive studies have also been done on representing majority functions of larger
fan-in by the majority of majority functions of lower fan-in
 \cite{amano18,engels20,kulikov19}.

Since it is empirically known that the number of trees affects
the prediction performance \cite{oshiro12},
it is of interest 
to theoretically study the trade-off between
the number of trees and the size of trees.
However, to our knowledge, there is almost no theoretical study
on the size and number of decision trees in a random forest.

Recently, Kumano and Akutsu studied embedding of random forests and
binary decision diagrams into layered neural networks with sigmoid,
ReLU, and similar activation functions \cite{kumano22}.
They also showed that 
$\ds{\Omega\left( \left( {\frac {2^n}{\sqrt{n}}}\right)^{2/(T+1)} \right)}$
nodes are needed to represent the majority function of $n$ variables 
using a random forest consisting of $T$ decision trees.
However, 
deriving an upper bound on the number of nodes in a random forest was 
left as an open problem,
without specifying the target function.
The main purpose of this paper is to partially answer this (a bit ambiguous)
question.
In this paper, we focus on the prediction phase of random forests
for binary classification problems with binary input variables.
That is, we consider binary classification problems for which
the decision is made by the majority vote of a bag of decision trees.
Furthermore, we consider a simple decision tree model in which
each internal node asks whether some input variable has value 1 or 0.
Since we focus on the prediction phase of a random forest,
we use ``random forest'' to mean 
a classifier whose decision is made according to
the majority vote of the decision trees in the forest, 
as in \cite{audemard22}.

We show that the majority function on $n=2m-1$ variables can be represented
by a bag of $T$ ($< n$) trees in which the maximum size of each tree is
$O(n^{((n-T)/2)+1})$, where $n-T$ is assumed to be a constant.
This result means that the size of the
transformed trees is polynomially bounded.
We also consider more generally 
the problem of transforming a given bag of $n$ trees
to a bag of $T$ ($< n$) trees. However, the approach 
that was used for the majority function cannot be 
directly applied to the general problem.
We, therefore, modify the general problem 
to the task of transforming a given forest of $n$ trees
to a forest with fewer trees  
that is allowed to make small errors, 
in the sense that the 
total weight of the input vectors that give inconsistent outputs
divided by the total weight of all input vectors is small
(i.e., we consider an arbitrary probability distribution on samples).
Assuming that $c$ and $K$ are fixed positive integers, we show that 
for a bag of $2m-1$ decision trees each with size at most $r$,
there exists a bag of $2m-1-2c$ decision trees such that
the size of each tree
is $O(r^{(2K+11)^c})$ and the error is
at most ${\frac{c}{2K}}$.
In addition, as part of the derivation of our results on the majority function,
we show that the $k$-out-of-$n$ (majority) function, which is often referred to
as a threshold function, can be represented by a bag of $n$ decision trees 
each of size $O(n^{|m-k|+1})$.

\section{Preliminaries}

In this paper we consider simple decision trees over binary domains
and bags of simple decision trees,
defined as follows
(see also Fig.~\ref{fig:dt-rf}(A) and (B)).
\begin{definition}
A \emph{decision tree} is a rooted tree in which
a query is assigned to each internal 
node. The tree is simple if all queries are
limited to asking the value of
one input variable from a set of Boolean input variables 
$X=\{x_1,\ldots,x_n\}$,
and the leaves carry labels that are either 0 or 1.
The output for a given sample $\textbf{x}=(x_1,\ldots,x_n)$
is given by the label of the leaf that is reached from the root
by answering the queries.
The size of a decision tree is the number of nodes in the tree.
\end{definition}

\begin{definition}
A \emph{bag of decision trees} is a collection of $N=2M-1$ decision trees
$T_1,\ldots,T_N$
on a set of Boolean variables
$X=\{x_1, \ldots, x_n\}$
which on a given input returns the decision arrived at by a majority of the trees.
\end{definition}

In the coming two
sections we focus on the representation  
of the majority function, $\textit{Maj}$,
which returns 1
if the majority of its $n=2m-1$ (Boolean) variables has value 1,
by a bag of trees. 
That task is very easy when using a bag of $n$ trees: 
simply set $T_i=x_i$.
Representing it with a bag of less than $n$ trees takes more doing,
which is one of the main subjects of this paper.

A type of bag that will be especially useful in the following is 
one that uses $N$ trees to implement the
$k$-out-of-$n$ function 
$\textit{Maj}(k;x_1,\ldots , x_n)$:
the bag returns 1 
if and only if at least $k$ among its $n$ input values are 1.
We denote this type of bag by 
$C_{N}(k;x_1,\ldots , x_n)$, or more concisely where warranted by $C_{N}(k;\textbf{x})$. $C$ is 
mnemonic for 
\emph{Choose bag}, as it returns 1 if and only if it is possible to 
choose $k$ input variables that have the value 1.
For literals $z_1, z_2 \cdots, z_n$,
$z_1 z_2 \cdots z_n$ represents $z_1 \land z_2 \land \cdots \land z_n$.

\begin{figure}
\begin{center}
\includegraphics[width=12cm]{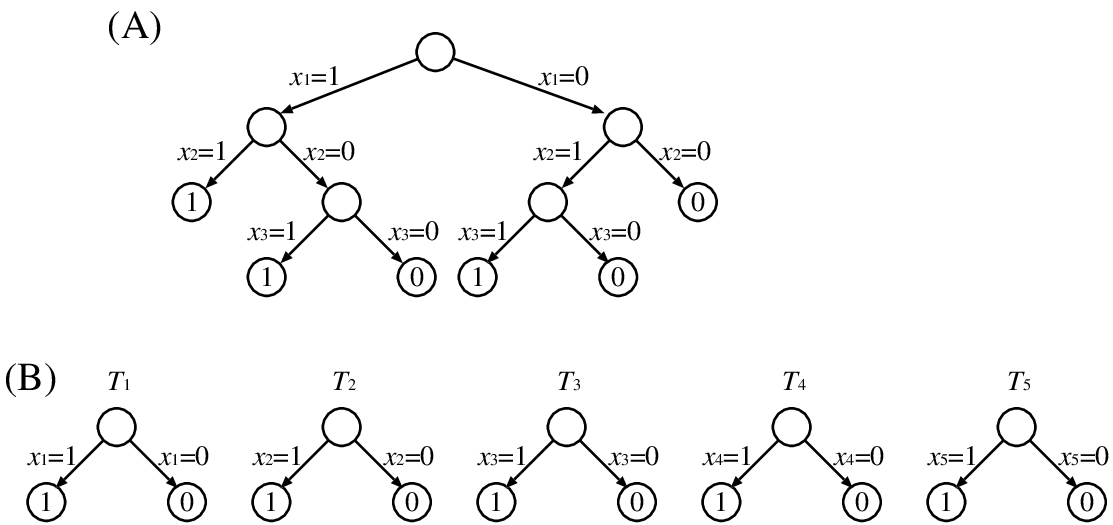}
\caption{(A) Decision tree representing the majority function on 3 variables,
and (B) Random forest (bag of trees) representing the majority function on
5 variables.}
\label{fig:dt-rf}
\end{center}
\end{figure}

\section{Bag of trees for the $k$-out-of-$n$ function}
\label{sec:k-of-n-2}

In this section, we present a method for constructing
a bag of decision trees for representing the $k$-out-of-$n$ function.
This method will be used as a kind of subroutine for
the main problem.

We begin by defining two types of (sub)trees that will appear in the
implementation. We define these trees by describing which inputs they 
decide to accept, or by the corresponding Boolean formula.
% Such a description is easily converted into a tree whose decision nodes 
% are of the form $``x_i=1?"$
\begin{enumerate}
\item The (sub)tree $L_{\ell}(i)$ accepts \textbf{b} if and only if $b_i=0$ and 
it is the leftmost, or second leftmost, or \ldots, the $\ell$-th leftmost 0 in \textbf{b}.
I.e., there are at most
$1\leq p\leq \ell$ indices  
$i_1 \leq \ldots <i_p  = i$ such that 
$b_{i_1} =\ldots =b_{i_{p}}=0$.

For example, suppose that $n=3$ (or larger).
Then, we have
\begin{eqnarray*}
L_1(3) & = & x_1 \land x_2 \land \lnon{x_3},\\
L_2(3) & = & (x_1 \land x_2 \land \lnon{x_3}) \lor
(\lnon{x_1} \land x_2 \land \lnon{x_3}) \lor
(x_1 \land \lnon{x_2} \land \lnon{x_3}).
\end{eqnarray*}
Note that all of $x_1,\ldots,x_{b_i}$ appear in each conjunction.
Then, we can construct a decision tree by simply branching
on $x_1$ at the root,  on $x_2$ at the depth 1 nodes,
on $x_3$ at the depth 2 nodes, and so on.
Since it is enough to consider at most 
\begin{eqnarray*}
\sum_{i=0}^{\ell-1} \binom{n}{i} & \leq & O(n^{\ell-1})
\end{eqnarray*}
combinations
and the number of nodes from the root to each leaf is at most $n+1$.
$L_{\ell}(i)$ can be represented by a decision tree
of size $O(n^{\ell})$,
where the size of a decision tree is the number of its nodes.

\item $I_{\ell}(i)$ is a (sub)tree which accepts an input \textbf{b}
if and only if $b_i=1$ and \textbf{b} 
contains at least $\ell$ values 1 after $b_i$, i.e.
$b_{i_1} =\ldots =b_{i_{\ell}}=1$ for some 
$i <i_1 <\ldots <i_{\ell}\leq 2m-1$ and
$b_j=0$ for the other $j$s such that $i < j < i_{\ell}$.
It will be helpful to visualize the interval $[i_j,i_{j+1}]$
as a ``ones-interval'', with $I_{\ell}(i)$ verifying that,
starting at $i$, there are $\ell$ contiguous ones-intervals in the input.
We will call these $\ell$ contiguous ones-intervals a 
\textit{``ones-interval of length $\ell$''}.
	
For example, suppose that $n=5$.
Then, we have
\begin{eqnarray*}
I_1(2) & = & (x_2 \land x_3) \lor (x_2 \land \lnon{x_3} \land x_4) \lor
(x_2 \land \lnon{x_3} \land \lnon{x_4} \land x_5),\\
I_2(2) & = & (x_2 \land x_3 \land x_4) \lor (x_2 \land x_3 \land \lnon{x_4} \land x_5) \lor
(x_2 \land \lnon{x_3} \land x_4 \land x_5).
\end{eqnarray*}
Note that each term contains consecutive variables (i.e.,
$x_i,x_{i+1},\ldots,x_j$ for some $j$) beginning from the same variable
$x_i$.
Then, we can construct a decision tree by simply branching
on $x_i$ at the root,  on $x_{i+1}$ at the depth 1 nodes,
on $x_{i+2}$ at the depth 2 nodes, and so on.
Since it is enough to consider at most $\binom{n}{\ell}$ terms
and the number of nodes from the root to each leaf is at most $n+1$.
$I_{\ell}(i)$ can be represented by a decision tree of size $O(n^{\ell+1})$.
\end{enumerate}

As a function of $k$, the trees of $C_{n}(k;\textbf{x})$ are as follows:
\begin{enumerate}
\item For $k=m-\ell$ with $1\leq \ell \leq m-1$,
\begin{center}
\begin{tabular}{|l|l|l|l|l|l|}
\hline	
$T_1$       & $\cdots$   & $T_{\ell}$ 
& $T_{\ell+1}$ 
& $\cdots$  &$T_{2m-1}$\\
\hline
\hline

$1$  & $ \cdots $ &  $1	$ 
& $x_{\ell+1}\lor L_{\ell}(\ell+1)$   
& $\cdots$                  &  $x_{2m-1}\lor L_{\ell}(2m-1)$\\
\hline		
\end{tabular}
\end{center}
Note that $T_i$ for $i \leq \ell$ corresponds to $x_i \lor L_{\ell}(i)$.
However, each of them is always 1, and thus 1 is given in this table.
\item For $k=m$, 
\begin{center}
\begin{tabular}{|l|l|l|}
\hline	
$T_1$       & $\cdots$     &$T_{2m-1}$\\
\hline
\hline
			
$x_1$  & $ \cdots $ &  $x_{2m-1}$\\
\hline		
\end{tabular}
\end{center}
\item For $k=m+\ell$ with $1\leq \ell \leq m-1$,
\begin{center}
\begin{tabular}{|l|l|l|l|l|l|}
\hline	
$T_1$       & $\cdots$   & $T_{2m-\ell-1}$ 
& $T_{2m-\ell}$ 
& $\cdots$  &$T_{2m-1}$\\
\hline
\hline
			
$I_{\ell}(1)$  & $ \cdots $ &  $I_{\ell}(2m-\ell-1)$ 
& 0   
& $\cdots$                  &  0\\
\hline		
\end{tabular}
\end{center}

\end{enumerate}

For the first case, $T_i$ for $i \geq \ell+1$ has 
To represent this Boolean function by a decision tree,
it is enough to add a new root (corresponding to $x_i$)
to the decision tree for $L_{\ell}(i)$,
where the new root has two children: one is a leaf with label 1,
the other is the root of the tree for $L_{\ell}(i)$.
 
\begin{theorem}
For fixed $k$,
$C_n(k;\textbf{x})$ can be represented by a bag of decision trees
each of which has size $O(n^{|m-k|+1})$.
\label{thm:k-outof-N2}
\end{theorem}
\begin{proof}
Let $C^n(k;\textbf{x})$ be the bag of decision trees constructed as above.
For simplicity,
we also represent the result of the majority vote of this bag on 
input  $\textbf{b}$  by $C^n(k;\textbf{b})$.
We need therefore to 
prove that $C^n(k;\textbf{b})=1$ if and only there are 
at least $k$ 1's in \textbf{b}. This is obviously true
in case $k=m$.

\begin{enumerate}
\item $k=m-\ell$ with $1\leq \ell \leq m-1$.

Suppose first that \textbf{b} contains at least $k$ 1's.
If \textbf{b} contains at least \textit{m} 1's then 
$C^{n}(k;\textbf{b})$ returns 1 per its definition.
Suppose then that \textbf{b} contains $m-p$ 1's,
with $1\leq p \leq \ell$. Noting that \textbf{b} contains more than
$\ell$ 0's, adjoin the indices of the $\ell$ leftmost ones to the set of 
of indices of the $m-p$ 1's. The resulting set has size $m-p+\ell\geq m$. Moreover, for each index $i$ in the set $x_i\lor L_{\ell}(i)$ has value
1 so that $T_i$ returns 1. Hence $C^{n}(k;\textbf{b})=1$. 

To prove the converse assume that $m$ of the trees in 
$C^{n}(k;\textbf{b})$ return 1 on input $\textbf{b}$.
There are therefore at least $m-\ell$ indices $i_j>\ell$
such that $b_{i_j}\lor L_{\ell}(i_j)$ for $i_j$.
Denote by $p_0$ and $q_0$ the number of indices $i$ such that 
$L_{\ell}(i)=1$ for $i\leq \ell$ and $i>\ell$, respectively,
and by $p_1$ and $q_1$ the number of indices $i$ such that 
$b_i=1$
for $i\leq \ell$ and $i>\ell$, respectively.
Then $p_0+p_1=\ell$, $p_0+q_0\leq \ell$ 
(by the definition of $L$), and $q_0+q_1 \geq m-\ell$. 
Hence the number of 1's in \textbf{b}, $p_1+q_1$, satisfies
$$p_1+q_1=\ell-p_0+q_1 \geq q_0+q_1 \geq m-\ell=k.$$

\item $k=m+\ell$ with $1\leq \ell \leq m-1$.

If \textbf{b} contains at least $m+\ell$ 1's at indices
$i_1<\ldots <i_{m+\ell}$, then  
each $i_j$, $1\leq j \leq m$, is 
the left endpoint of a ones-interval of length $\ell$ in \textbf{b}.
Therefore, for $1\leq j \leq m $, $I_{\ell}(i_j) =1$ and 
 $T_{i_j}$ returns 1. 
It follows that
$C^{n}(k;\textbf{b})=1$.

To prove the converse assume that $m$ of the trees in 
$C^{n}(k;\textbf{b})$ return 1 on input $\textbf{b}$.
There are therefore $i_1 < \cdots < i_m$ such that 
$I_{\ell}(i_j)=1$, and in particular $b_{i_j}=1$. These, together 
with the $\ell$ additional bits that make $I_{\ell}(i_m)=1$
show that $\textbf{b}$ contains $k$ 1's.
\end{enumerate}

Finally, it is seen from the discussions before this theorem
that the size of each decision tree is
\begin{eqnarray*}
O(n^{\ell}) = O(n^{m-k}) & & \mbox{for $1 \leq k < m$}.\\
O(1) & & \mbox{for $k=m$},\\
O(n^{\ell+1}) = O(n^{k-m+1}) & & \mbox{for $m < k \leq n$}.
\end{eqnarray*}
Therefore, for each $k$,
the size of each decision tree is $O(n^{|m-k|+1})$.
\end{proof}

Note that for small $k$,
the size given by this theorem is not better than 
that by a naive construction method,
in which 
all $k$ combinations of variables are represented by $T_1$
(as in Figure~\ref{fig:dt-rf}(A)),
each of $T_2,\ldots,T_m$ represents 1, 
each of $T_{m+1},\ldots,T_{2m-1}$ represents 0.
Analogously, for large $k$ (close to $n$),
the size given by this theorem is not better than 
that by a naive construction method.

It should also be noted that the case of $k=m-\ell$ and
the case of $k=m+\ell$ are not symmetric even if
we switch 0 and 1.
0 and 1 are interchanged.
Suppose $m=3$ (i.e., $n=2m-1=5$) and $k=m+1=4$.
In order to get a positive output,
at least four input variables should have value 1.
Here,
we switch 0 and 1 in the input variables and
consider the case of $k=m-1=2$.
In order to get a positive output,
at least 2 variables should have the switched value 1.
This means that at most 3 variables have the switched value 0,
corresponding to that at most 3 variables have the original value 1.
This example suggests that 
the cases of $k=m-\ell$ and $k=m+\ell$ are not symmetric.

\section{Majority function on $2m-1$ variables by $2m-1-2c$ trees}
\label{sec:majority}

%%TA.23.12.09>
%In this section, we present our main result.
In this section, we present one of our main results.
%%<TA.23.12.09

\begin{theorem}
For any positive integer constant $c$,
the majority function on $n=2m-1$ variables can be represented 
by a bag of $2m-1-2c$ decision trees where the maximum size of each tree is 
$O(n^{c+1})$.
\label{thm:majority}
\end{theorem}
\begin{proof}
To illustrate the construction we look first at the simplest case, $c=1$.
Extending our notation a bit, the majority function on $n=2m-1$ variables can be represented by a bag $B_{2m-3}$ containing $n-2=2m-3$ trees as follows:
\begin{eqnarray}\label{eqn:reduced}
B_{2m-3}(x_1,x_2,\ldots,x_n)=	& ((x_1 x_2) \land C_{2m-3}(m-2;x_3,\ldots,x_n)) & \lor \nonumber\\
	& ((x_1 \lnon{x_2}) \land C_{2m-3}(m-1;x_3,\ldots,x_n)) & \lor \nonumber\\
	& ((\lnon{x_1} x_2) \land C_{2m-3}(m-1;x_3,\ldots,x_n)) & \lor \nonumber\\
	& ((\lnon{x_1} \lnon{x_2}) \land C_{2m-3}(m;x_3,\ldots,x_n)).
\end{eqnarray}
To describe $B_{2m-3}$ explicitly let $T^{-1}_i$, $T^0_i$, and $T^{+1}_i$ be the $i$th tree in
the bags of $2m-3$ trees representing
$C_{2m-3}(m-2;x_3,\ldots,x_n)$,
$C_{2m-3}(m-1;x_3,\ldots,x_n)$,
and
$C_{2m-3}(m;x_3,\ldots,x_n)$,
respectively.
We also use $T^{-1}_i$, $T^0_i$, and $T^{+1}_i$ 
to denote the Boolean functions corresponding to these trees.

Then the $i$th tree of $B_{2m-3}$, $T_i$, is
\begin{eqnarray*}
T_i=	((x_1 x_2) \land T^{-1}_i) \lor
	((x_1 \lnon{x_2}) \land T^0_i) \lor
	((\lnon{x_1}x_2) \land T^0_i) \lor
	((\lnon{x_1} \lnon{x_2}) \land T^{+1}_i).
\end{eqnarray*}

For example, in case $n=5$, we get $3$ decisions trees 
whose Boolean functions are:
\begin{eqnarray*}
	((x_1 x_2) \land 1) \lor
	((x_1 \lnon{x_2}) \land x_3) \lor
	((\lnon{x_1} x_2) \land x_3) \lor
	((\lnon{x_1} \lnon{x_2}) \land (x_3 x_4 \lor x_3 \lnon{x_4} x_5)),\\
	((x_1 x_2) \land (x_4 \lor x_3 \lnon{x_4})) \lor
	((x_1 \lnon{x_2}) \land x_4) \lor
	((\lnon{x_1} x_2) \land x_4) \lor
	((\lnon{x_1} \lnon{x_2}) \land (x_4 x_5)),\\
	((x_1 x_2) \land (x_5 \lor x_3 x_4 \lnon{x_5}) \lor
	((x_1 \lnon{x_2}) \land x_5) \lor
	((\lnon{x_1} x_2) \land x_5) \lor
	((\lnon{x_1} \lnon{x_2}) \land 0).
\end{eqnarray*}
The first two 
levels of the tree $T_i$, therefore, branch on the values of $x_1$ and $x_2$, respectively, 
with the 4 trees $T^{-1}_i$, $T^0_i$, $T^0_i$, and $T^{+1}_i$ grafted onto the 4 outcomes.
According to Section~\ref{sec:k-of-n-2} the sizes of $T^{-1}_i$, $T^0_i$, and $T^{+1}_i$ are $O(n)$, $O(1)$, and $O(n^2)$, respectively,
so that the theorem holds in this case.

In order to generalize the above discussion to a general constant $c$
we define $P(r,s)$ for $r\geq s$ and given variables $x_1,\ldots,x_n$ by
\begin{eqnarray*}
	P(r,s) & = &
	\bigvee \{ z_1 z_2 \cdots z_r \mid
	\mbox{($z_i = x_i$ or $z_i = \lnon{x_i}$ for all $i$)}\\
	& & 
	\mbox{and (exactly $s$ of the literals $z_i$ are of the form $z_i = \lnon{x_i}$ ) }  \},
\end{eqnarray*}
where $P$ is mnemonic for \emph{prefix}. To  illustrate, $P(2,0)=x_1x_2$, 
$P(2,1)=x_1 \lnon{x_2} \lor \lnon{x_1}x_2$, and 
$P(2,2)=\lnon{x_1}\lnon{x_2}$. Using this notation to generalize  
the discussion above, we see that
the majority function on $n=2m-1$ variables can
be represented by

\begin{eqnarray*}
\bigvee_{s=0}^{2c}
\Bigl(
P(2c,s) \land C_{2m-1-2c}(m-2c+s;x_{2c+1},x_{2c+2},\ldots,x_n)
\Bigr).
\end{eqnarray*}
Here, we note:
\begin{itemize}

\item 
$C_{2m-1-2c}(m-c-\ell;x_{2c+1},\ldots,x_n)$ corresponds
to the case of $k=m-\ell$ of Section~\ref{sec:k-of-n-2},
\item 
$C_{2m-1-2c}(m-c;x_{2c+1},\ldots,x_n)$ corresponds
to the case of $k=m$ of Section~\ref{sec:k-of-n-2},
\item 
$C_{2m-1-2c}(m-c+\ell;x_{2c+1},\ldots,x_n)$ corresponds
to the case of $k=m+\ell$ of Section~\ref{sec:k-of-n-2}.
\end{itemize}
It is seen from Theorem~\ref{thm:k-outof-N2} that each 
$C_{2m-1-2c}(m-2c+s;x_{2c+1},\ldots,x_n)$
can be represented by a bag of $2m-1-2c$ trees each of which
has $O(n^{c+1})$ nodes.
Let $T^{-\ell}_i$, $T^0_i$, and $T^{+\ell}_i$ 
denote the $i$th tree (and also the corresponding Boolean function)
in the bag of $2m-1-2c$ trees representing
$C_{2m-1-2c}(m-c-\ell;x_{2c+1},\ldots,x_n)$,
$C_{2m-1-2c}(m-c;x_{2c+1},\ldots,x_n)$,
and
$C_{2m-1-2c}(m-c+\ell;x_{2c+1},\ldots,x_n)$,
respectively. Then, as in the case of $c=1$,
the majority function on $2m-1$ variables can be
represented by the majority function on $2m-1-2c$ decision trees
in which the 
$i$th decision tree, $T_i$,
represents the following Boolean function:
\begin{eqnarray*}
T_i=
\bigvee_{\ell=-c}^{c} \Bigl( P(2c,c+\ell) \land  T^{\ell}_i \Bigr).
\end{eqnarray*}
Constructing $T_i$ as in the case $c=1$ yields a tree of size 
$O(2^{2c}n^{c+1})$ which is $O(n^{c+1})$ for constant $c$.
\end{proof}

\section{Reducing the number of trees in a 		
general	bag of trees while allowing small errors}
\label{sec:bag}

Unfortunately,
the construction given in Section \ref{sec:majority} cannot be 		
used to reduce the number of trees in a general random forest (while keeping the tree size polynomial)
because conjunction or disjunction of $n$ decision trees each with size $s$
may need a decision tree of size $\Omega(s^n)$, as discussed in
Section~\ref{sec:basic-idea}.
Therefore, we modify the problem by still asking for a reduction in the 
number of decision trees in a forest but permitting the transformed forest to make a small number of classification errors,
all the while keeping the size of each tree polynomial in $n$.

\subsection{Basic idea}
\label{sec:basic-idea}

Before 
delving into the details of the construction, we put our finger on the reason why extending Theorem~\ref{thm:majority}
to general random forests is difficult, and then present the central idea for
circumventing this difficulty.

It is tempting to deal with the problem of reducing the number of trees in a general bag of trees by modifying the approach that worked for the majority function. After all, the function implemented by the bag is a majority function albeit of trees rather than of variables. 
%But that difference is precisely wherein the rub lies. 
Consider, for example, the case $c=1$. 
Given a bag of trees $B_{2m-1}(\xvec)=\{t_1(\xvec),\ldots,t_{2m-1}(\xvec)\}$ we would like to construct a bag with only $2m-3$ trees, $B_{2m-3}$, 
using equation (\ref{eqn:reduced}) while replacing the variables 
$x_i$ on the right hand side by the trees $t_i(\xvec)$:
\begin{eqnarray}\label{eqn:reducedg}
	B_{2m-3}(\xvec)=	
	& (\ (t_1(\xvec) t_2(\xvec)) \land C_{2m-3}(m-2;t_3(\xvec),\ldots,t_{2m-1}(\xvec))\ ) & \lor \nonumber\\
	& (\ (t_1(\xvec) \lnon {t_2(\xvec)}\ ) \land C_{2m-3}(m-1;t_3(\xvec),\ldots,t_{2m-1}(\xvec))\ ) & \lor \nonumber\\
	& (\ (\lnon{t_1(\xvec)} t_2(\xvec)) \land 
	C_{2m-3}(m-1;t_3(\xvec),\ldots,t_{2m-1}(\xvec))\ ) & \lor \nonumber\\
	& (\ (\lnon{t_1(\xvec)} \lnon{t_2(\xvec)}) \land C_{2m-3}(m;t_3(\xvec),\ldots,t_{2m-1}(\xvec))\ ) .
\end{eqnarray}
Following the discussion of Section \ref{sec:k-of-n-2},
a tree in the choose bag 
$C_{2m-3}(m-2;t_3(\xvec),\ldots,t_{2m-1}(\xvec))$, for example,
is of the form $t_{i+2}(\xvec) \lor L_{1}(i+2)$ -
where $L_{1}(i+2)$ is 
a composition of $i$ trees, 
$L_1(i+2)  =\wedge_{j=3}^{i+1} (t_j(\xvec) \land \lnon{t_{i+2}(\xvec)})$
	\footnote{The negation of a decision tree can
	be obtained by mutually exchanging the 0 and 1 labels
	assigned to leaves.}.
However, the number of nodes in $L_1(i)$ can grow exponentially with the number of trees.

\begin{proposition}
Let $T_1$ and $T_2$ be decision trees with $s_1$ and $s_2$ nodes,
respectively.
Then, the conjunction (resp., disjunction) of $T_1$ and $T_2$
can be represented by a decision tree with $O(s_1 s_2)$ nodes.
\label{prop:conjunction}
\end{proposition}
\begin{proof}

To construct the decision tree for the conjunction of $T_1$ and $T_2$
replace each leaf of $T_1$ that has the label 1 by a copy of $T_2$.
Clearly the resulting decision tree
represents the conjunction of $T_1$ and $T_2$,
and it has $O(s_1 s_2)$ nodes.
\end{proof}

By repeatedly applying this proposition,
the conjunction of $n$ decision trees each having $O(s)$ nodes
can be represented by a decision tree of size $O(s^n)$.
As far as we know
there is no construction method that requires $o(s^n)$ size.
Therefore, a
simple modification of the construction would
yield exponential size decision trees even for the case 
of $c=1$.

To prevent this exponential buildup in the size of the trees
%%AM.23.12.24>
%we apply the construction of Section \ref{sec:majority} 
%to the first $K$ trees only, 
%while the remaining trees correspond to just single variables.
%To indicate this modification we replace $B_{2m-3}$ and $C_{2m-3}$ by 
% $\hat{B}_{2m-3}$ and $\hat{C}_{2m-3}$:
we apply the construction of Section \ref{sec:majority} 
to only $K$ carefully chosen trees, 
while the remaining trees correspond to just single variables. 
For simplicity of exposition assume for now that the \textit{first} 
$K$ trees were chosen, and denote these modified bags by
 $\hat{B}_{2m-3}$ and $\hat{C}_{2m-3}$,
 so that equation \ref{eqn:reducedg} becomes:
%%<AM.23.12.24 
\begin{eqnarray}\label{eqn:reducedg2}
	\hat{B}_{2m-3}(\xvec)=	
	& (\ (t_1(\xvec) t_2(\xvec)) \land \hat{C}_{2m-3}(m-2;t_3(\xvec),\ldots,t_{2m-1}(\xvec))\ ) & \lor \nonumber\\
	& (\ (t_1(\xvec) \lnon {t_2(\xvec)}\ ) \land \hat{C}_{2m-3}(m-1;t_3(\xvec),\ldots,t_{2m-1}(\xvec))\ ) & \lor \nonumber\\
	& (\ (\lnon{t_1(\xvec)} t_2(\xvec)) \land 
	\hat{C}_{2m-3}(m-1;t_3(\xvec),\ldots,t_{2m-1}(\xvec))\ ) & \lor \nonumber\\
	& (\ (\lnon{t_1(\xvec)} \lnon{t_2(\xvec)}) \land \hat{C}_{2m-3}(m;t_3(\xvec),\ldots,t_{2m-1}(\xvec))\ ) .
\end{eqnarray}
%%AM.23.12.24>
\begin{remark}
How to choose the $K$ trees is an issue that we will address 
in the next section. If it turns out that the chosen trees are not the first ones we make them the first ones by reordering.
\end{remark}
\begin{example}\label{ex:Kbag}
Consider the reduction of a bag with $n=2m-1=11$ trees to a bag of size 
$9$, with $K=3$. 
The trees in $\hat{B}_{9}$ are listed
in Table~\ref{tab:approx-majority}.
Here $L_1(i)=t_1 \land \cdots \land t_{i-1} \land \lnon{t_{i}}$,
and 
%%AM.24.02.01>
%$I_1(i)=(t_{i} \land t_{i+1}) \lor (t_{i} \land \lnon{t_{i+1}} \land t_{i+2}) \lor
%(t_{i}\land \lnon{t_{i+1}} \land \lnon{t_4} \land t_5)$
%Note that the value of $L_1(3)$ is relevant only for $\xvec$ such that
%$t_1(\xvec) =t_2(\xvec)=1$, and then 
%$t_3(\xvec) \lor L_1(3)$ is always 1.
%and
%$I_1(5)$ is always 0.
$I_1(i)=t_{i} \land (t_{i+1} \lor t_{i+2} \lor \cdots \lor t_{K+2})$,
for $3\leq i \leq K+1$, with $I_1(K+2)=0$.
Note that the value of $L_1(3)$ is relevant only for $\xvec$ such that
$t_1(\xvec) =t_2(\xvec)=1$, and then 
$t_3(\xvec) \lor L_1(3)$ is always 1.

%%<AM.24.02.01

\begin{table}[h]
	\begin{center}
		\caption{Representation of 
			the majority function on 11 trees by a bag with 9 trees,
						$\hat{B}_{9}$,
			that allows errors, 
			where $K=3$.}
		\label{tab:approx-majority}
		\begin{tabular}[19cm]{|l||l|l|l|l|l|l|l|l|l|}
			\hline
			prefix &
			$T_1$ & $T_2$ & $T_3$ & $T_4$ & $T_5$ & $T_6$ & $T_7$& $T_8$ & $T_9$\\
			\hline
			$t_1 t_2$
			& $t_3 \lor$
			& $t_4 \lor$
			& $t_5 \lor$
			& $t_6$ 
			& $t_7$ 
			& $t_8$ 
			& $t_9$
			& $t_{10}$ 
			& $t_{11}$ \\
			& $L_1(3)$ 
			& $L_1(4)$
			& $L_1(5)$
			&
			&
			&
			&
			& 
			& \\
			\hline
			$\lnon{t_1} t_2$ (or $t_1 \lnon{t_2}$) &
			$t_3$ &
			$t_4$ &
			$t_5$ &
			$t_6$ &
			$t_7$ &
			$t_8$ &
			$t_9$ & 
			$t_{10}$ & 
			$t_{11}$  \\
			\hline
			$\lnon{t_1 t_2}$ &
			$I_1(3)$ &
			$I_1(4)$ &
			$I_1(5)$ &
			$t_6$ &
			$t_7$ &
			$t_8$ &
			$t_9$ & 
			$t_{10}$ & 
			$t_{11}$\\
			\hline
		\end{tabular}
	\end{center}
\end{table}

For example, an error  
with $t_1(\xvec) =t_2(\xvec)=1$
occurs when
$\tvec(\xvec)=(1,1,1,1,1,1,0,0,0,0,0)$
because only four of the nine transformed trees are 1
(i.e., $\hat{{\bf T}}(\xvec) = (1,1,1,1,0,0,0,0,0)$),
and an error with $t_1(\xvec) =t_2(\xvec)=0$ occurs when
$\tvec(\xvec)=(0,0,0,0,0,0,1,1,1,1,1)$
because five of the nine transformed trees are 1
(i.e., $\hat{{\bf T}}(\xvec) = (0,0,0,0,1,1,1,1,1)$).

\newpage
As a prelude to the
error analysis detailed in the next section let us examine when $\hat{B}_{9}(\xvec)$ returns an incorrect answer given that $t_1(\xvec) =t_2(\xvec)=1$.
Denote by $S$ the set of $\xvec$ such that $t_1(\xvec)=t_2(\xvec)=1$,
and by $S^1$, and $S^0$, those $\xvec\in S$ for which the answer returned by
$\hat{B}_{9}(\xvec)$ is the same answer, respectively not the same answer, 
as the one returned by the original bag $B_{11}(\xvec)$.

According to the definition of $L_1(i)$,
if on input $\xvec$ at least one of the trees $t_3,t_4,t_5$ returns 0
then the number of 1's returned by the trees in 
$\hat{B}_{9}(\xvec)$ is
1 plus the number of 1's among 
$t_3(\xvec) ,\ldots,t_{11}(\xvec)$.
Therefore, if at most 2 out of the trees 
$t_3(\xvec) ,\ldots,t_{11}(\xvec)$ return 1 
(so that $B_{11}(\xvec)=\hat{B}_{9}(\xvec)=0$),
or if at least 4 return 1  
(so that $B_{11}(\xvec)=\hat{B}_{9}(\xvec)=1$).
%(taking into account that $t_1(\xvec)=t_2(\xvec)=1$),
Moreover, even if 3 of these trees return 1,  
$\hat{B}_{9}(\xvec)$ still returns the correct answer, $1$, unless 
$t_3(\xvec)=t_4(\xvec)=t_5(\xvec)=1$. 
\end{example}

The probability that $\hat{B}_{9}$ returns an erroneous answer 
appears, therefore, to be small. However, this issue requires a careful analysis
because the configurations of tree values need not be equally likely,
even if the inputs $\textbf{x}$ are uniformly distributed. 
This brings to the forefront the question of how to choose the $K$ 
trees to which to
apply the construction of Section \ref{sec:majority}. 
We distill the salient features of the question in the following 
Problem $P$.
Its solution will be detailed in the next section.

\textbf{Problem} $P$: 
\textit{Let ${\cal C}=\{\ 0,1\}^n $ be the set of potential tree configurations,
and given $p<n$ let ${\cal P}$ be the set of all $\textbf{c}\in {\cal C}$ such that $\textbf{c}$ contains exactly $p$ 1's. Given $q<p$ and a set of 
$q$ indices, $J$, let ${\cal P}_J\subset {\cal P}$ be the set 
of all configurations 
$\textbf{c}$ with the property that $\textbf{c}[j]=1$ for all $j\in J$
(and $\textbf{c}$ contains exactly $p$ 1's). The problem
is to find a $J^*$ such that ${\cal P}_{J^*}$ has a small probability.}

\subsection{Details of the construction}
%%AM.23.12.24>
We begin with some notations and definitions. Given a bag of decision trees, ${\cal T} = (T_1,\ldots,T_n)$, 
on the set of all binary vectors of length $l$, 
${\cal X} = \{\xvec_1,\ldots,\xvec_N\}$, $N=2^l$,
let $\textbf{T}(\xvec)=(T_1(\xvec),T_2(\xvec),\ldots,T_n(\xvec))$.
The probability distribution on the set of potential tree 
configurations ${\cal C}=\{\ 0,1\}^n $ induced by a probability distribution $D$ on ${\cal X}$ is defined as follows.
\begin{definition}
	Let $D(\xvec)$ be a probability distribution on ${\cal X}$.
	The probability that the bag ${\cal T}$
	has the configuration $\cvec\in {\cal C}$ is
	\begin{eqnarray*}
		\mu(\cvec) & = & \sum_{\xvec \in {\cal X}: \textbf{T}(\xvec)=\cvec} D(\xvec).
	\end{eqnarray*}	 
\end{definition}
Note that $\sum_{\cvec} \mu(\cvec) = \sum_{\xvec} D(\xvec) = 1$.

Next we define our error measure.
Denote by ${\cal T}(\xvec)$ the majority decision of ${\cal T}$ on $\xvec$,
i.e. ${\cal T}(\xvec)=\textit{Maj} \circ \textbf{T}(\xvec)$. Our interest is in
finding a bag with a smaller number of trees using only the 
information contained in $\textbf{T}(\xvec)$. In other words,
we will consider only functions on ${\cal X}$ of the form 
${\cal F}={\cal G}\circ \textbf{T}(\xvec)$. 

\begin{definition}
	Let $D(\xvec)$ be an arbitrary probability distribution on ${\cal X}$.
	Given a Boolean function ${\cal F}={\cal G}\circ \bf{T}(\xvec)$ on ${\cal X}$ we define its 
	error with respect to ${\cal T}$ by
	\begin{eqnarray}\label{eqn:errormeasure}
		err({\cal F},{\cal T}) & = & 
		\sum_{\bf{c} \in {\cal C}}
		\sum_{\xvec : {\bf T}(\xvec)=\bf{c}
		\mbox{ and }{\cal F}(\xvec) \neq {\cal T}(\xvec)} D(\xvec)\\\nonumber
		& = & \sum_{\bf{c} \in {\cal C}:{\cal G}(\cvec) \neq {\textit{Maj} }(\cvec) } 
		\mu(\cvec)
	\end{eqnarray}
\end{definition}
Note that $0 \leq err({\cal F},{\cal T}) \leq 1$.

We use these definitions to address Problem $P$ posed in the previous subsection: for given $q$ is it possible to choose a set of $q$ trees to which to
apply the construction of Section \ref{sec:majority} so as to 
guarantee a relatively small probability of error. We denote set of the indices of those trees $J^*$.

\begin{lemma}
Given $q<p<n$, denote by ${\cal P}=\{\cvec_1,\ldots,\cvec_r\}$ the set of all $\textbf{c}\in {\cal C}$ such that $\textbf{c}$ contains exactly $p$ 1's, 
$r = \binom{n}{p}$. Let ${\cal J}$
denote the set of all sets $J$ consisting of exactly $q$ indices 
in the range $1,\ldots, n$. For $J\in {\cal J}$ let 
${\cal P}_J\subseteq {\cal P}$ 
be the set of all configurations 
$\bf{c}\in {\cal P}$ with the property that $\textbf{c}[j]=1$ for all $j\in J$.
	Then there exists a $J^*\in {\cal J} $ such that 
	$\mu({\cal P}_{J^*}) \leq  \left({\binom{p}{q}}/{\binom{n}{q}}\right) \mu({\cal P})$.
	\label{lem:perm-lemma}
\end{lemma}
\begin{proof}
Given $J$ define $\mu_J({\bf c}_i)$ for ${\bf c}_i \in {\cal P}$ by 
	\begin{eqnarray*}
	\mu_J({\bf c}_i)= \left\{
	\begin{array}{ll}
	\mu({\bf c}_i), & \mbox{if $\cvec_i[j]=1$ for all $j\in J$},\\
		0, &\mbox{otherwise}.
	\end{array}
	\right.
\end{eqnarray*}
Note that for fixed $i$
$$\sum_{J:\ |J|=q} \mu_J({\bf c}_i) = \binom{p}{q} \mu({\bf c}_i),$$ 
so that 
$$\sum_{i=1}^{r} \sum_{J:\ |J|=q} \mu_J({\bf c}_i) = 
\sum_{i=1}^{r} \binom{p}{q} \mu({\bf c}_i) = 
\binom{p}{q} \mu({\cal P}).$$
On the other hand, 
	\begin{eqnarray*}
		\sum_{i=1}^{r} \sum_{J:|J|=q} \mu_J({\bf c}_i) & = & 
		\sum_{J:|J|=q} \sum_{i=1}^{r} \mu_J({\bf c}_i)~=~
    	\sum_{J:|J|=q} \mu({\cal P}_J).
	\end{eqnarray*}
It follows that 
\begin{eqnarray*}
	\sum_{J:|J|=q} \mu({\cal P}_J)=\binom{p}{q} \mu({\cal P}).
\end{eqnarray*}
There are $\binom{n}{q}$ summands $J$s on the left-hand side, hence
	there is set $J^*$ of cardinality $q$ such that
	\begin{eqnarray*}
	\mu({\cal P}_{J^*}) & \leq 
	\frac{\binom{p}{q}}{\binom{n}{q}}   \mu({\cal P}).
	\end{eqnarray*}

\end{proof}

%%<AM.23.12.24

\begin{theorem}

Let ${\cal T}=(t_1,\ldots,t_n)$ be a bag of $n=2m-1$ 
decision trees on ${\cal X}$, and let the size of the largest
$t_i$ be 
%%AM.23.12.24>
%$r$.
$s$.
%%<AM.23.12.24
Then, for any positive integer constant $K$,
there exists a bag of $2m-3$ decision trees on ${\cal X}$, 
$\hat{B}_{2m-3}$,
such that each of its trees has size $O(
%%AM.23.12.24>
%r
s
%%<AM.23.12.24
^{2K+11})$ and 
$err({\hat{B}_{2m-3}},{\cal T}) \leq {\frac 1 {2^K}}$.
\label{thm:approx-minus2}
\end{theorem}
\begin{proof}
Let $t_i(\xvec)$ be the output of $t_i$ on input vector $\xvec$. 
To construct $\hat{B}_{2m-3}$ according to equation (\ref{eqn:reducedg2}), we need to first define  
$\hat{C}_{2m-3}(k;t_3(\xvec),\ldots,t_{2m-1}(\xvec))$
for $k=m-2,m-1,m$.

\begin{enumerate}
	\item $\hat{C}_{2m-3}(m-2;t_3(\xvec),\ldots,t_{2m-1}(\xvec))=
	\{T^{-}_1(\xvec), \ldots, T^{-}_{2m-3}(\xvec) \}$, 
	where \\
	$T^{-}_i(\xvec)=\twopartdef {t_{2+i}(\xvec)\lor L_1(2+i)} {i=1,\ldots,K} {t_{2+i}(\xvec)} {i=K+1,\ldots,2m-3}$ \quad  \\
 with $L_1(j)=t_3(\xvec) \land t_4(\xvec) \land \cdots \land t_{j-1}(\xvec) \land \lnon{t_j(\xvec)}$;
	\item $\hat{C}_{2m-3}(m-1;t_3(\xvec),\ldots,t_{2m-1}(\xvec))
	=\{t_3(\xvec),\ldots,t_{2m-1}(\xvec)\}$;
	\item $\hat{C}_{2m-3}(m;t_3(\xvec),\ldots,t_{2m-1}(\xvec))=
	\{T^{+}_1(\xvec), \ldots T^{+}_{2m-3}(\xvec)\}$, where\\
	$T^{+}_i(\xvec) = 
	%%AM.24.02.01>
%	\twopartdef {I_1(2+i)} {i=1,\ldots,K} 
%	            {t_{2+i}(\xvec)} {i=K+1,\ldots,2m-3}$ \quad \\
	\threepartdef {I_1(2+i)} {i=1,\ldots,K-1} 
	{0} {i=K}
	{t_{2+i}(\xvec)} {i=K+1,\ldots,2m-3}$ \quad \\
	%%<AM.24.02.01
	with $I_1(j)=t_j(\xvec) \land (t_{j+1}(\xvec) \lor \cdots \lor t_{2+K}(\xvec))$.
\end{enumerate}

Thus $\hat{B}_{2m-3} =\{T_1(\xvec),\ldots,T_{2m-3}(\xvec)\}$ is defined by
\begin{eqnarray*}
%%TA.23.12.09>
%T_i(\xvec) & = & (t_1(\xvec) \land {t_2}(\xvec) \land T_i^{-}(\xvec)) \lor (t_1(\xvec) \land \lnon{t_2}(\xvec) \land t_{2+i}(\xvec)) \lor
%(\lnon{t_1}(\xvec) \land t_2(\xvec) \land t_{2+i}(\xvec)) \lor\\
%&& (\lnon{t_1(\xvec)} \land \lnon{t_2}(\xvec)\land T_i^{+}(\xvec)) ,
T_i(\xvec) & = & (t_1(\xvec) \land {t_2}(\xvec) \land T_i^{-}(\xvec)) \lor (t_1(\xvec) \land \lnon{t_2(\xvec)} \land t_{2+i}(\xvec)) \lor
(\lnon{t_1(\xvec)} \land t_2(\xvec) \land t_{2+i}(\xvec)) \lor\\
&& (\lnon{t_1(\xvec)} \land \lnon{t_2(\xvec)}\land T_i^{+}(\xvec)) ,
%%<TA.23.12.09
\end{eqnarray*}
and the value of $\hat{B}_{2m-3}$ on an input $\xvec$, $\hat{B}_{2m-3}(\xvec)$, is the majority vote
of its decision trees, i.e.
it is 1 if and only if $m-1$ or more of the $T_i(\xvec)$s
are satisfied.

We turn next to the analysis of $err({\hat{B}_{2m-3}},{\cal T})$.

%%AM.23.12.31>
%\begin{eqnarray*}
%W^{(j_1,j_2)} & = & \sum_{\bvec: 
%	(\bvec[1],\bvec[2])=(j_1, j_2)
%} w(\bvec)
%\end{eqnarray*}
%where $w(\bvec)=\sum_{\xvec \in {\cal X}: \tvec(\xvec)=\bvec} D(\xvec)$.
%Note that 
%$W^{(0,0)}+W^{(0,1)}+W^{(1,0)}+W^{(1,1)} = 1$.

Let $h(\xvec)$ and $\hat{h}(\xvec)$ be the number of decision trees in ${\cal T}$, and $\hat{B}_{2m-3}$ respectively, 
satisfied by $\xvec$, i.e. $h(\xvec) = \sum_{i=1}^{2m-1} t_i(\xvec)$,
and $\hat{h}(\xvec)= \sum_{i=1}^{2m-3} T_i(\xvec)$.
Also, $h^{-}(\xvec)$ and $h^{+}(\xvec)$ are 
the number of decision trees $T_i^{-}$, and  $T_i^{+}$, respectively, which return 1 on $\xvec$, i.e.
$h^{-}(\xvec) = \sum_{i=1}^{2m-
%%AM.23.12.21>
%1
3
%%<AM.AM.23.12.21
} T_i^{-}(\xvec)$,
and ${h}^{+}(\xvec)= \sum_{i=1}^{2m-3} T_i^{+}(\xvec)$.

We break our analysis up into 3 cases.
\begin{description}
\item [Case $t_1(\xvec) =t_2(\xvec)=1$:]\ \\
	Then $h(\xvec)=2+\sum_{i=3}^{2m-1} t_i(\xvec)$, and 
	\begin{eqnarray*}
		\hat{h}(\xvec)=h^{-}(\xvec) & = & \min\left(1+\sum_{i=3}^{2+K} t_i(\xvec),K \right) + \sum_{i=3+K}^{2m-1} t_i(\xvec) \\
		&=& \min\left(h(\xvec)-1,h(\xvec)-2+K-  \sum_{i=3}^{K+2} t_i(\xvec) \right) 
	\end{eqnarray*}
On the basis of these expressions we distinguish four possibilities:
\begin{itemize}
	\item if $h(\xvec) \geq m+1$
	then $\hat{h}(\xvec) \geq m-1$, so that 
	${\cal T}(\xvec)=\hat{B}_{2m-3}(\xvec)=1$.
	\item if $h(\xvec) \leq m-1$
	then $\hat{h}(\xvec) \leq m-2$, so that 
	 ${\cal T}(\xvec)=\hat{B}_{2m-3}(\xvec)=0$.
	\item if  $h(\xvec) = m$, and in particular ${\cal T}(\xvec)=1$, then
	\begin{itemize}
	\item if $\sum_{i=3}^{K+2} t_i(\xvec)<K$,
	then $\hat{h}(\xvec)= m-1$ so that also 
	$\hat{B}_{2m-3}(\xvec)=1$.
	\item if $\sum_{i=3}^{K+2} t_i(\xvec)=K$,
	then $\hat{h}(\xvec)= m-2$ so that 
	$\hat{B}_{2m-3}(\xvec)=0$.
		\end{itemize}
	\end{itemize}
Only the last case results in an erroneous value of $\hat{B}_{2m-3}(\xvec)$, i.e. an error occurs only when exactly $m-2$ among 
$(t_3(\xvec) ,\cdots ,t_{2m-1}(\xvec))$ are 1, and
$(t_3(\xvec) , \cdots ,t_{K+2}(\xvec)) = (1, 1 ,\cdots ,1)$,
which corresponds to the case of  $n=2m-3$, $p=m-2$ and $q=K$ of 
Lemma~\ref{lem:perm-lemma}.

Therefore, the Lemma ensures that there exists a permutation 
of $t_3,\ldots,t_n$ such that
the error in this case is at most
\begin{eqnarray*}
	\left({\binom{m-2}{K}}/{\binom{2m-3}{K}}\right) \mu({\cal C}_{(1,1)})
	\leq{\frac 1 {2^K}} \mu({\cal C}_{(1,1)}),
\end{eqnarray*}
where we have used the notation. 
\begin{eqnarray*}
	{\cal C}_{(\epsilon_1,\epsilon_2)}=\{{\bf c}\in {\cal C}\mid {\bf c}[1]=\epsilon_1,{\bf c}[2]=\epsilon_2  \}, {\rm with }\ \epsilon_i\in \{0,1\}
\end{eqnarray*}
Note that 
$\mu({\cal C}_{(0,0)})+\mu({\cal C}_{(0,1)})+
 \mu({\cal C}_{(1,0)})+\mu({\cal C}_{(1,1)}) = 1$.

\item [Case ($t_1(\xvec) =1$ and $t_2(\xvec)=0$) or 
            ($t_1(\xvec) =0$ and $t_2(\xvec)=1$):]\ \\
Then $h(\xvec)=1+\hat{h}(\xvec)=1+\sum_{i=3}^{2m-1} t_i(\xvec)$, so that 
$h(\xvec)\geq m$ if and only if $\hat{h}(\xvec)\geq m-1$.

\item [Case $t_1(\xvec) =t_2(\xvec)=0$:]\ \\
Then $h(\xvec)=\sum_{i=3}^{2m-1} t_i(\xvec)$, and 
\begin{eqnarray*}
	\hat{h}(\xvec)=h^{+}(\xvec) & = & \max\left(-1+\sum_{i=3}^{2+K} t_i(\xvec),0 \right) + \sum_{i=3+K}^{2m-1} t_i(\xvec) \\
	&=& \max\left(h(\xvec)-1,h(\xvec)- \sum_{i=3}^{K+2} t_i(\xvec) \right).
\end{eqnarray*}

On the basis of these expressions we distinguish four possibilities:
\begin{itemize}
	\item if $h(\xvec) \geq m$
	then $\hat{h}(\xvec) \geq m-1$, so that  
	${\cal T}(\xvec)=\hat{B}_{2m-3}(\xvec)=1$.
	\item if $h(\xvec) \leq m-2$
	then $\hat{h}(\xvec) \leq m-2$, so that 
	${\cal T}(\xvec)=\hat{B}_{2m-3}(\xvec)=0$.
	\item if  $h(\xvec) = m-1$, and in particular ${\cal T}(\xvec)=
	%%AM.23.12.24>
%	1
	0
	%%<AM.23.12.24
	$, then 
	\begin{itemize}
		\item if $\sum_{i=3}^{K+2} t_i(\xvec)=0$,
		then $\hat{h}(\xvec)= m-1$ so that 
			%%AM.23.12.24>
		%	also
		$\hat{B}_{2m-3}(\xvec)=1\neq {\cal T}(\xvec)$.
		%%<AM.23.12.24 
		$\hat{B}_{2m-3}(\xvec)=1$.
		\item if $\sum_{i=3}^{K+2} t_i(\xvec)>0$,
		then $\hat{h}(\xvec)< m-1$ and 
		$\hat{B}_{2m-3}(\xvec)=0$.
	\end{itemize}
\end{itemize}
\end{description}
In this case, therefore, an error occurs only when
%%AM.23.12.24>
%In this case, therefore, an error occurs only when
%$|\{i \mid (i \geq 3) \land (t_i(\xvec)=0)\}|=m-2$
%and
%$t_3(\xvec) t_4(\xvec) \cdots t_{3+K-1}(\xvec) = 0 0 \cdots 0$.
%This corresponds to the case of $L=m-2$ and $H=2m-3$ of
%Lemma~\ref{lem:perm-lemma},
%by exchanging the roles of 0 and 1.
exactly $m-2$ among 
$(t_3(\xvec) ,\cdots ,t_{2m-1}(\xvec))$ are 0, and
$(t_3(\xvec) , \cdots ,t_{K+2}(\xvec)) = (0, 0 ,\cdots ,0)$.
This condition corresponds to the case of  $n=2m-3$, $p=m-2$ and $q=K$ of 
Lemma~\ref{lem:perm-lemma} by exchanging the roles of 0 and 1.
%%<AM.23.12.24
Hence, the Lemma ensures the existence of a permutation 
of $t_3,\ldots,t_n$ such that
the error in this case is at most
\begin{eqnarray*}
	%%AM.23.12.24>
%	\left({\binom{m-2}{K}}/{\binom{2m-3}{K}}\right) W^{(0,0)}& = &
%	\left( {\frac {m-2}{2m-3}} \cdot {\frac {m-3}{2m-4}} \cdots
%	{\frac {m-1-K}{2m-2-K}} \right) W^{(0,0)}\\
%	& \leq & {\frac 1 {2^K}} W^{(0,0)}.
\left({\binom{m-2}{K}}/{\binom{2m-3}{K}}\right) \mu({\cal C}_{(0,0)})
\leq{\frac 1 {2^K}} \mu({\cal C}_{(0,0)}).
	%%<AM.23.12.24
\end{eqnarray*}

Since the three cases  are mutually independent,
the permutations 
	%%AM.23.12.24>
%can be done independently.
used in the first and the last case need not be the same.
	%%<AM.23.12.24
We conclude that there exists a bag of $2m-3$ decision trees whose error is at most
\begin{eqnarray*}
		%%AM.23.12.24>
{\frac 1 {2^K}} \mu({\cal C}_{(1,1)})+{\frac 1 {2^K}} \mu({\cal C}_{(0,0)})
% \leq {\frac 1 {2^K}} \left( W^{(1,1)}+W^{(0,0)} \right) 
\leq  {\frac 1 {2^K}} .
	%%<AM.23.12.24
\end{eqnarray*}

Finally, we analyze the size of the resulting trees.
The size can be upper-bounded as $O(
%%AM.23.12.24>
%r
s
%%<AM.23.12.24
^{Y+1})$
if the number of $\land$ and $\lor$ operators for representing
a Boolean function over $t_1,\ldots,t_n$
corresponding to each resulting tree is $Y$.
For the case of $t_1 t_2$, the number of operators is
at most $2+K$.
For the case of $t_1 \lnon{t_2}$ (resp., $\lnon{t_1} t_2$),
the number of operators is 2.
For the case of $\lnon{t_1} \lnon{t_2}$,
the number of operators is at most $1+K$.
Hence, the total number of operators is
$(2+K)+2+2+(1+K)+3 = 2K+10$.
It is seen from Proposition~\ref{prop:conjunction} that
the resulting upper bound of the size of each tree is $O(
%%AM.23.12.24>
%r
s
%%<AM.23.12.24
^{2K+11})$.
\end{proof}

It may be possible to develop a similar construction procedure
for the case of $c>1$ by modifying the construction in
Theorem~\ref{thm:majority}.
However, its probabilistic analysis would be quite difficult.
Instead, we use a simple recursive procedure.
First, we reduce the number of trees to $2m-3$ using
Theorem~\ref{thm:approx-minus2}.
Then, we reduce the number of trees
to $2m-5$ using Theorem~\ref{thm:approx-minus2} again.
We repeat this procedure until the number of trees becomes
$2m-1-2c$.

\begin{theorem}
Let ${\cal T}$ be a bag of $n=2m-1$ decision trees on a set of variables
$X=\{x_1,\ldots,x_{n'}\}$, where the size of each decision tree is at most $r$.
Then, for any positive integer constants $c$ and $K$,
there exists a bag of $2m-1-2c$ decision trees on $X$
such that the size of each bag is $O(r^{(2K+11)^c})$ and 
the error is at most ${\frac c {2^K}}$.
\label{thm:approx-general}
\end{theorem}
\begin{proof}
First, we analyze the size of decision trees in a bag.
At the first iteration,
the size increases from $r$ to $O(r^{2K+11})$ from
Theorem~\ref{thm:approx-minus2}.
Then, this increased size corresponds to the next $r$.
Since $K$ is a constant, at the second iteration,
the size increases to
\begin{eqnarray*}
O((r^{2K+11})^{2K+11}) & = & O(r^{(2K+11)^2}).
\end{eqnarray*}
Since we repeat this procedure $c$ times and $c$ is a constant,
the final size is $O(r^{(2K+11)^c})$.

Next, we analyze the error.
At the first iteration,
the error is at most ${\frac 1 {2^K}}$ from Theorem~\ref{thm:approx-minus2}.
Then, we set the weight of the erroneous samples to be 0 and
scaling up the weights of the other samples so that the total weight becomes
1.
Then, at the second iteration,
the error caused by the second iteration is at most ${\frac 1 {2^K}}$,
where the actual error is not greater than this because
the weights are scaled up.
Therefore, the total error is at most ${\frac 2 {2^K}}$.
Since we repeat this procedure $c$ time,
the final error is at most ${\frac c {2^K}}$.
\end{proof}

\section{Concluding remarks}

In this paper, we studied the trade-off between the number of nodes and
the number of trees in a bag of decision trees for representing a
given bag of $n$ decision trees.
As a main result, we showed that 
%%AM.23.12.31>
%if $n-T$ is a constant, 
%there exists a bag of $T$ trees of polynomial size that represents
%the majority function of $n$ variables.
%If $n-T$ is not a constant,
the majority function of $n$ variables, that naturally can be represented
by a bag of $n$ single node trees, can also be represented
by a bag of $n'=n-c$ trees, each of polynomial size, in case $c$ is a given
constant. If $c$ is not a constant,
%given a constant $c$, 
%there exists a bag of $n'=n-c$ trees, each of polynomial size, that represents
%the majority function of $n$ variables.
%If $c$ is not a constant,
%%<AM.23.12.31
the derived size is exponential, which is consistent with
an exponential lower bound given in \cite{kumano22}.
However, the gap between the upper and lower bounds is still large.
Therefore, the narrowing of the gap is left as an open problem.

We also considered the general case of representing a given bag of $n$
decision trees using a bag of $n'=n-c$ decision trees, for constant $c$.
For this case we showed that if the bag is permitted to make classification 
errors with a small probability, then such a bag
%% TA.24.02.03>
%with polynomial size decision trees does indeed exist.
with polynomial-size decision trees does indeed exist.
%% <TA.24.02.03
The question of whether there exists a bag of $n'$ 
	%%AM.24.02.01>
	polynomial-size
%%<AM.24.02.01
decision trees
that makes no errors, is left as an open problem. 

\bibliographystyle{plain}

\bibliography{rfmajoarxiv}

\begin{thebibliography}{10}

\bibitem{amano18}
Kazuyuki Amano and Masafumi Yoshida.
\newblock Depth two (\emph{n}-2)-majority circuits for \emph{n}-majority.
\newblock {\em {IEICE} Trans. Fundam. Electron. Commun. Comput. Sci.},
  101-A(9):1543--1545, 2018.

\bibitem{audemard22}
Gilles Audemard, Steve Bellart, Louenas Bounia, Fr{\'{e}}d{\'{e}}ric Koriche,
  Jean{-}Marie Lagniez, and Pierre Marquis.
\newblock Trading complexity for sparsity in random forest explanations.
\newblock In {\em Proceedings of Thirty-Sixth {AAAI} Conference on Artificial
  Intelligence (AAAI-2022)}, pages 5461--5469, 2022.

\bibitem{biau12}
G{\'{e}}rard Biau.
\newblock Analysis of a random forests model.
\newblock {\em J. Mach. Learn. Res.}, 13:1063--1095, 2012.

\bibitem{breiman01}
Leo Breiman.
\newblock Random forests.
\newblock {\em Mach. Learn.}, 45(1):5--32, 2001.

\bibitem{chistopolskaya22}
Anastasiya Chistopolskaya and Vladimir~V. Podolskii.
\newblock On the decision tree complexity of threshold functions.
\newblock {\em Theory Comput. Syst.}, 66(6):1074--1098, 2022.

\bibitem{delgado14}
Manuel~Fern{\'{a}}ndez Delgado, Eva Cernadas, Sen{\'{e}}n Barro, and
  Dinani~Gomes Amorim.
\newblock Do we need hundreds of classifiers to solve real world classification
  problems?
\newblock {\em J. Mach. Learn. Res.}, 15(1):3133--3181, 2014.

\bibitem{engels20}
Christian Engels, Mohit Garg, Kazuhisa Makino, and Anup Rao.
\newblock On expressing majority as a majority of majorities.
\newblock {\em {SIAM} J. Discret. Math.}, 34(1):730--741, 2020.

\bibitem{goldmann92}
Mikael Goldmann, Johan H{\aa}stad, and Alexander~A. Razborov.
\newblock Majority gates {VS.} general weighted threshold gates.
\newblock {\em Comput. Complex.}, 2:277--300, 1992.

\bibitem{kulikov19}
Alexander~S. Kulikov and Vladimir~V. Podolskii.
\newblock Computing majority by constant depth majority circuits with low
  fan-in gates.
\newblock {\em Theory Comput. Syst.}, 63(5):956--986, 2019.

\bibitem{kumano22}
So~Kumano and Tatsuya Akutsu.
\newblock Comparison of the representational power of random forests, binary
  decision diagrams, and neural networks.
\newblock {\em Neural Comput.}, 34(4):1019--1044, 2022.

\bibitem{lorenzen19}
Stephan~Sloth Lorenzen, Christian Igel, and Yevgeny Seldin.
\newblock On pac-bayesian bounds for random forests.
\newblock {\em Mach. Learn.}, 108(8-9):1503--1522, 2019.

\bibitem{magniez16}
Fr{\'{e}}d{\'{e}}ric Magniez, Ashwin Nayak, Miklos Santha, Jonah Sherman,
  G{\'{a}}bor Tardos, and David Xiao.
\newblock Improved bounds for the randomized decision tree complexity of
  recursive majority.
\newblock {\em Random Struct. Algorithms}, 48(3):612--638, 2016.

\bibitem{oshiro12}
Thais~Mayumi Oshiro, Pedro~Santoro Perez, and Jos{\'{e}}~Augusto Baranauskas.
\newblock How many trees in a random forest?
\newblock In {\em Proceedings of 8th International Conference on Machine
  Learning and Data Mining in Pattern Recognition (MLDM 2012)}, volume 7376 of
  {\em Lecture Notes in Computer Science}, pages 154--168, 2012.

\bibitem{siu95}
Kai-Yeung Siu, Vwani Roychowdhury, and Thomas Kailath.
\newblock {\em Discrete Mathematics of Neural Networks, Selected Topics}.
\newblock Prentice Hall, 1995.

\bibitem{testa19}
Eleonora Testa, Mathias Soeken, Luca~Gaetano Amar{\`{u}}, Winston Haaswijk, and
  Giovanni~De Micheli.
\newblock Mapping monotone boolean functions into majority.
\newblock {\em {IEEE} Trans. Computers}, 68(5):791--797, 2019.

\bibitem{valiant84}
Leslie~G. Valiant.
\newblock A theory of the learnable.
\newblock {\em Commun. {ACM}}, 27(11):1134--1142, 1984.

\bibitem{xu21}
Haoyin Xu, Kaleab~A. Kinfu, Will LeVine, Sambit Panda, Jayanta Dey, Michael
  Ainsworth, Yu-Chung Peng, Madi Kusmanov, Florian Engert, Christopher~M.
  White, Joshua~T. Vogelstein, and Carey~E. Priebe.
\newblock When are deep networks really better than decision forests at small
  sample sizes, and how?, 2021.

\bibitem{zhou17}
Zhi-Hua Zhou and Ji~Feng.
\newblock Deep forest: Towards an alternative to deep neural networks.
\newblock In {\em Proceedings of the Twenty-Sixth International Joint
  Conference on Artificial Intelligence, {IJCAI-17}}, pages 3553--3559, 2017.

\end{thebibliography}

\end{document}